\documentclass[nohyperref]{article}

\usepackage{microtype}
\usepackage{graphicx}
\usepackage{subfigure}
\usepackage{booktabs} 

\usepackage{hyperref}

\usepackage[accepted]{icml2022}

\usepackage{amsmath}
\usepackage{amssymb}
\usepackage{mathtools}
\usepackage{amsthm}

\usepackage[capitalize,noabbrev]{cleveref}

\newcommand{\rank}{\textrm{rank}}
\usepackage{multirow}
\usepackage{colortbl}

\theoremstyle{plain}
\newtheorem{theorem}{Theorem}[section]

\theoremstyle{definition}

\theoremstyle{remark}

\usepackage[textsize=tiny]{todonotes}
\graphicspath{{figures/}}

\icmltitlerunning{Geometry-Aware Unsupervised Domain Adaptation}

\begin{document}

\twocolumn[
\icmltitle{Geometry-Aware Unsupervised Domain Adaptation}



\icmlsetsymbol{equal}{*}

\begin{icmlauthorlist}
\icmlauthor{You-Wei Luo}{sysu}
\icmlauthor{Chuan-Xian Ren}{sysu}
\icmlauthor{Zi-Ying Chen}{sysu}
\end{icmlauthorlist}

\icmlaffiliation{sysu}{School of Mathematics, Sun Yat-Sen University}

\icmlcorrespondingauthor{C. X. Ren}{rchuanx@mail.sysu.edu.cn}

\icmlkeywords{Machine Learning, ICML}

\vskip 0.3in
]



\printAffiliationsAndNotice{}  

\begin{abstract}
    Unsupervised Domain Adaptation (UDA) aims to transfer the knowledge from the labeled source domain to the unlabeled target domain in the presence of dataset shift. Most existing methods cannot address the domain alignment and class discrimination well, which may distort the intrinsic data structure for downstream tasks (e.g., classification). To this end, we propose a novel geometry-aware model to learn the transferability and discriminability simultaneously via nuclear norm optimization. We introduce the domain coherence and class orthogonality for UDA from the perspective of subspace geometry. The domain coherence will ensure the model a larger capacity for learning separable representations, and class orthogonality will minimize the correlation between clusters to alleviate the misalignment. So, they are consistent and can benefit from each other. Besides, we provide a theoretical insight into the norm-based learning literature in UDA, which ensures the interpretability of our model. We show that the norms of domains and clusters are expected to be larger and smaller to enhance the transferability and discriminability, respectively. Extensive experimental results on standard UDA datasets demonstrate the effectiveness of our theory and model.
\end{abstract}

\section{Introduction}

\noindent With the development of deep learning, computer vision and pattern recognition tasks, such as image classification and object detection, have been rapidly advanced. However, applying a model trained on source domain directly to a new environment (target domain) often results in significant performance degradation. This phenomenon is mainly caused by dataset shift~\cite{shimodaira2000improving}, i.e., there is a change in data distributions of different domains. As shown in Figure \ref{fig:problem_illustration}(a), the source and target subspaces are different but correlated from a geometric perspective. Domain adaptation (DA) can map data with different distributions of source domain and target domain into a shared representation space so that the shift problem can be alleviated. However, in practical scenarios, the data in new environment (target domain) are usually unlabeled. Annotating the target domain is a time-consuming and laborious task. As a more general problem, Unsupervised DA (UDA) is proposed to transfer the domain-invariant knowledge to a target domain without labels.

\begin{figure}[t]
\vskip 0.2in
\begin{center}
\centerline{\includegraphics[width=0.95\linewidth,trim=10 2 10 3,clip]{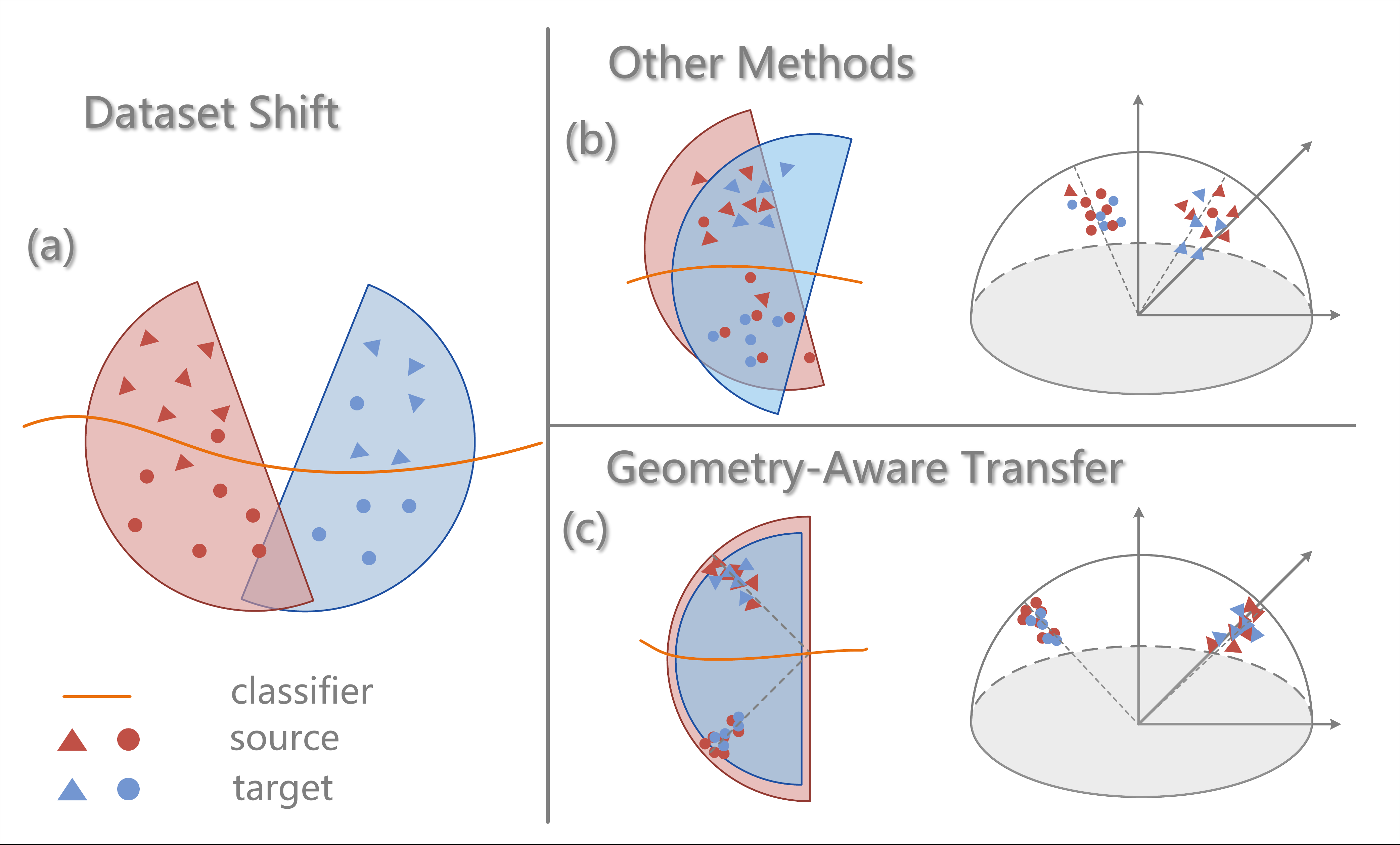}}
\caption{Problem illustration. \textbf{(a)}: The source domain and target domain are nearly linearly independent when dataset shift exists. \textbf{(b)}: Other methods of domain alignment do not make full use of geometric structure, so the clusters are potentially linearly dependent. \textbf{(c)}: The geometry-aware transfer learns a domain coherent subspace where the domains share a same basis, and a class orthogonal structure where the clusters are uncorrelated and maximum separated. These ensure both transferability and discriminability.}
\label{fig:problem_illustration}
\end{center}
\vskip -0.2in
\end{figure}

Existing UDA methods usually train the model to minimize the empirical risk on the source domain, and then use some approaches to reduce the discrepancy between source domain and target domain in an alternative manner. There are many approaches to reduce the cross-domain discrepancy such as Maximum mean discrepancy (MMD) \cite{borgwardt2006integrating,tzeng2014deep,long2015learning}, which directly measures the marginal discrepancy between domains; Joint MMD (JMMD)~\cite{long2017deep}, which models the joint distributions across domains as the tensor products in Hilbert spaces; Adversarial adaptation \cite{ganin2015unsupervised,tang2020discriminative,gu2020spherical} based on Generative adversarial network (GAN)~\cite{goodfellow2014generative}, which generates representations that are domain indistinguishable to achieve the domain invariance; Optimal Transport (OT)~\cite{courty2016optimal,zhang2019optimal,luo2021conditional}, which builds a transport problem across domains to reduce the Wasserstein distance; norm-based learning \cite{xu2019larger,cui2020towards}, which establishes a maximization problem with different norms to learn transferability. However, most of the UDA models are mainly designed for domain alignment, which cannot address class discrimination well. Besides, if the geometric structures of domains are learned insufficiently, the optimization procedure will be inevitably error-prone. It will distort the intrinsic data structures used for classification and degrade the model's performance in the target domain. As Figure \ref{fig:problem_illustration}(b), those methods are not aware of the geometric structures and may learn the correlated clusters, which induces the negative transfer. Also, \citet{chen2019transferability} empirically shown that some UDA methods tend to increase transferability at the cost of decreasing discriminability.

In this work, we propose a novel geometry-aware transfer model for enhancing transferability and discriminability. As high-dimensional data tend to have smaller intrinsic dimension \cite{qiu2015learning,lezama2018ole}, geometry-aware model tries to learn a compact representation space while maximum preserving the diversity of clusters and correlation of domains. As shown in Figure \ref{fig:problem_illustration}(c), the geometry-aware transfer learn the overlapped structures for different domains, where the linear dependence between clusters are minimized, i.e., orthogonality. Concretely, we propose the domain coherent and cluster orthogonal constraints for UDA problem, which provide an insight into geometric structures of domains and clusters via the rank of data matrix. To optimize the non-convex rank-based objective, we propose the surrogate nuclear norm-based model called \textbf{GE}ometry-aware \textbf{T}ransfer via nuclear norm optimization (GET). The contributions can be summarized as follows:

\begin{itemize}
\item We propose a geometry-aware model to capture intrinsic low-dimensional structures for downstream tasks, where the domains and clusters are modeled as the subspaces. It learns the cross-domain coherence and inter-class separability via the rank of subspaces, which ensures the transferability and discriminability.
\item We provide a theoretical insight into the norm-based learning literatures in UDA, which connects the discriminability/transferability enhancement with norm-based UDA models, and shows the possibility to learn both the abilities simultaneously.
\item Since the rank-based objective function is difficult to solve, we propose to optimize the rank-based model via surrogate nuclear norm, which ensures the interpretability and numerical stability of the proposed method. Extensive experiments and results demonstrate that GET can achieve new SOTA performance.
\end{itemize}

\section{Related Work}
\subsection{Unsupervised Domain Adaptation}
The purpose of UDA is to transfer knowledge from a labeled source domain to an unlabeled target domain. Existing UDA approaches measure the discrepancy between the source domain and target domain, then train a model to minimize this discrepancy. For example, \citet{tzeng2014deep} propose Deep Domain Confusion (DDC) by applying Maximum Mean Discrepancy (MMD) to the feature layer to reduce the domain discrepancy for the pre-trained AlexNet~\cite{krizhevsky2012imagenet}. Deep Adaptation Network (DAN)~\cite{long2015learning} extends the DDC by proposing a multi-layer and multiple kernel variant of MMD. Domain Adversarial Neural Network (DANN)~\cite{ganin2015unsupervised} trains a feature generator and a domain discriminator alternative to learn invariant representations. Adversarial-Learned Loss for Domain Adaptation (ALDA)~\cite{chen2020adversarial} constructs the confusion matrix with probability vectors of discriminator to adversarial model. \citet{na2021fixbi} propose the intermediate domains based on the mixup method, which can build a bridge between the source domain and target domain. Adaptive Feature Norm (AFN) \cite{xu2019larger} and Batch Nuclear-norm Maximization (BNM) \cite{cui2020towards} attempt to enhance the transferability by maximizing the matrix norm of batch data.

These methods mainly consider the domain alignment without integrating the class discrimination. Our method uses nuclear norm to align the subspaces of domains while minimizing the correlation between clusters.

\subsection{Class-Specific Learning}
Most of the UDA methods aim to learn domain invariant representations by aligning the domain globally. However, global domain alignment not only does not contribute to discriminability, but may also cause misclassification \cite{zhao2019learning}. Some recent works have considered class-specific adaptation for different application \cite{gong2016domain,ren2020learning,xu2021cross,luo2021conditional}. Since the target domain is unlabeled, these works rely on predicting pseudo-labeling~\cite{kang2019contrastive} or computing prototype representations of source and target classes~\cite{wang2020unsupervised}, and then the target domain samples are classified by the prototype of the target domain classes during the training process. Structure-preserving methods \cite{ren2019domain,xia2020structure} try to achieve the class-level transfer by matching the structure graphs across domains. However, as observed by \citet{chen2019transferability}, the discriminability may be decreased when the models only focus on enhancing transferability.

Compared with the norm-based methods \cite{xu2019larger,chen2019transferability,cui2020towards}, our model ensures the interpretability from a geometric perspective, and show the relation between the norm and discriminability/transferability mathematically. Besides, different from above works, we consider the class-specific adaptation from the perspective of geometry structure. It not only achieve the class-level transferability, but also enhance the cluster discriminability.

\section{Proposed Method}

For UDA, we are generally given source domain $\mathcal{D}_s=\{(x_i^s,y_i^s)\}_{i=1}^{n^s}$ and target domain $\mathcal{D}_t=\{x_i^t\}_{i=1}^{n^t}$ where $y$ is label with $k$ classes. $\mathbf{X}^s\in \mathbb{R}^{D\times n^s}$, $\mathbf{X}^t\in \mathbb{R}^{D\times n^t}$ are data matrices, $\mathbf{Y}^s\in \mathbb{R}^{k\times n^t}$ is one-hot label matrix, $\mathbf{X}=[\mathbf{X}^s,\mathbf{X}^t]\in\mathbb{R}^{D\times n}$ is the concatenation of $\mathbf{X}^s$ and $\mathbf{X}^t$. Let $\mathbf{X}^s_i\in\mathbb{R}^{D\times n^s_i}$ and $\mathbf{X}_i=[\mathbf{X}^s_i,\mathbf{X}^t_i]\in\mathbb{R}^{D\times n_i}$ be the data matrices of the $i$-th class and $i$-th source class, respectively. Note that the sample size satisfies that $n = n^s+n^t = n_1+n_2+\ldots+n_k$ and $n_i = n^s_i+n^t_i$. Generally, the superscript `$s/t$' means the domain, subscript `$i$' of capital (resp. lowercase) letter means the cluster (resp. sample). $\|\mathbf{X}\|_*$ is the nuclear norm of matrix $\mathbf{X}$, \textit{i.e.}, the sum of the singular values of $\mathbf{X}$. The goal of UDA is to find a domain-invariant representation space for both domains.

\subsection{Motivation}

We present the essential motivation based on the rank of matrix, which is closely related to the geometry structure of data. The smaller the rank, the higher the correlation between the column vectors (\textit{i.e.}, data) of the matrix. For the concatenation of domains $[\mathbf{X}^s,\mathbf{X}^t]$, a low-rank structure, where the source and target domains are correlated and linear dependent, is highly expected for enhancing transferability. In contrast, for the concatenation of clusters $[\mathbf{X}_1,\mathbf{X}_2,\ldots,\mathbf{X}_k]$, a full-rank structure is desired for learning independent cluster subspaces and enhancing discriminability. Our goal is to learn a compact representation space with enhanced transferability and discriminability, which ensure the effectiveness and interpretability for knowledge transfer model.

We first review some preliminary results on rank, which give an insight into the learning of geometry structure. It is well-known that for matrices $\mathbf{A}$, $\mathbf{B}$, there are
\begin{small}
\begin{equation*}
\rank ([\mathbf{A},\mathbf{B}])\leq \rank (\mathbf{A}) + \rank (\mathbf{B}),
\end{equation*}
\end{small}\noindent
with equality if and only if $\mathbf{A}$ and $\mathbf{B}$ are disjoint. It is clear that the inequity also holds for multiple matrices, \textit{i.e.},
\begin{small}
\begin{equation}\label{eq:rank-upper-bound}
\rank ([\mathbf{A}_1,\mathbf{A}_2,...,\mathbf{A}_k])\leq \sum_{i=1}^k \rank (\mathbf{A}_i).
\end{equation}
\end{small}\noindent
Besides, there is a lower bound of the rank of concatenation:
\begin{small}
\begin{equation}\label{eq:rank-lower-bound}
   \max\{\rank (\mathbf{A}) , \rank (\mathbf{B})\} \leq  \rank ([\mathbf{A},\mathbf{B}]),
\end{equation}
\end{small}\noindent
with equality if and only if the column space of one matrix is a subspace of the column space of another matrix.

Above analysis provides an motivation for learning transferability and discriminability via rank. Specifically, the concatenation of clusters $[\mathbf{X}_1,\mathbf{X}_2,\ldots,\mathbf{X}_k]$ is expected to reach the equality in upper bound Eq. \eqref{eq:rank-upper-bound}, where the bases of different classes are disjoint; the concatenation of domains $[\mathbf{X}^s,\mathbf{X}^t]$ is expected to reach the equality in lower bound Eq. \eqref{eq:rank-lower-bound}, where the subspaces of domains are overlapped and one domain can be represented by the basis of another.

Note that rank-based optimization is an NP-hard problem, then we consider the surrogate nuclear norm. Since nuclear norm is the convex envelop of rank over the unit ball of matrices, the upper-bound holds similarly as \cite{qiu2015learning}
\begin{small}
\begin{equation}\label{eq:nuclear-upper-bound}
\|[\mathbf{A},\mathbf{B}]\|_*\leq \|\mathbf{A}\|_*+\|\mathbf{B}\|_*,
\end{equation}
\end{small}\noindent
where the equality can be achieved as follows.
\begin{theorem}\label{thm:class_orthogonal}\cite{qiu2015learning}
Let $\mathbf{A}$ and $\mathbf{B}$ be matrices of the same row dimensions, we have
\begin{small}
\begin{equation*}
\|[\mathbf{A},\mathbf{B}]\|_*=\|\mathbf{A}\|_*+\|\mathbf{B}\|_*,
\end{equation*}
\end{small}
when the column spaces of $\mathbf{A}$ and $\mathbf{B}$ are orthogonal.
\end{theorem}

Orthogonality means the maximum separation, that is all principal angles between subspaces are equal to $\frac{\pi}{2}$. So nuclear norm ensures a strong property than rank where the bases are only linearly independent.

To learn a coherent domain geometry, the lower bound for nuclear norm and the condition of equality, which are yet unexplored, are indeed necessary. In the next, we will propose the rank/norm-based constraints and theoretically connect the norm-based learning with UDA.

\begin{figure}[t]
    \begin{center}
    \includegraphics[width=0.95\linewidth,trim=10 2 10 3,clip]{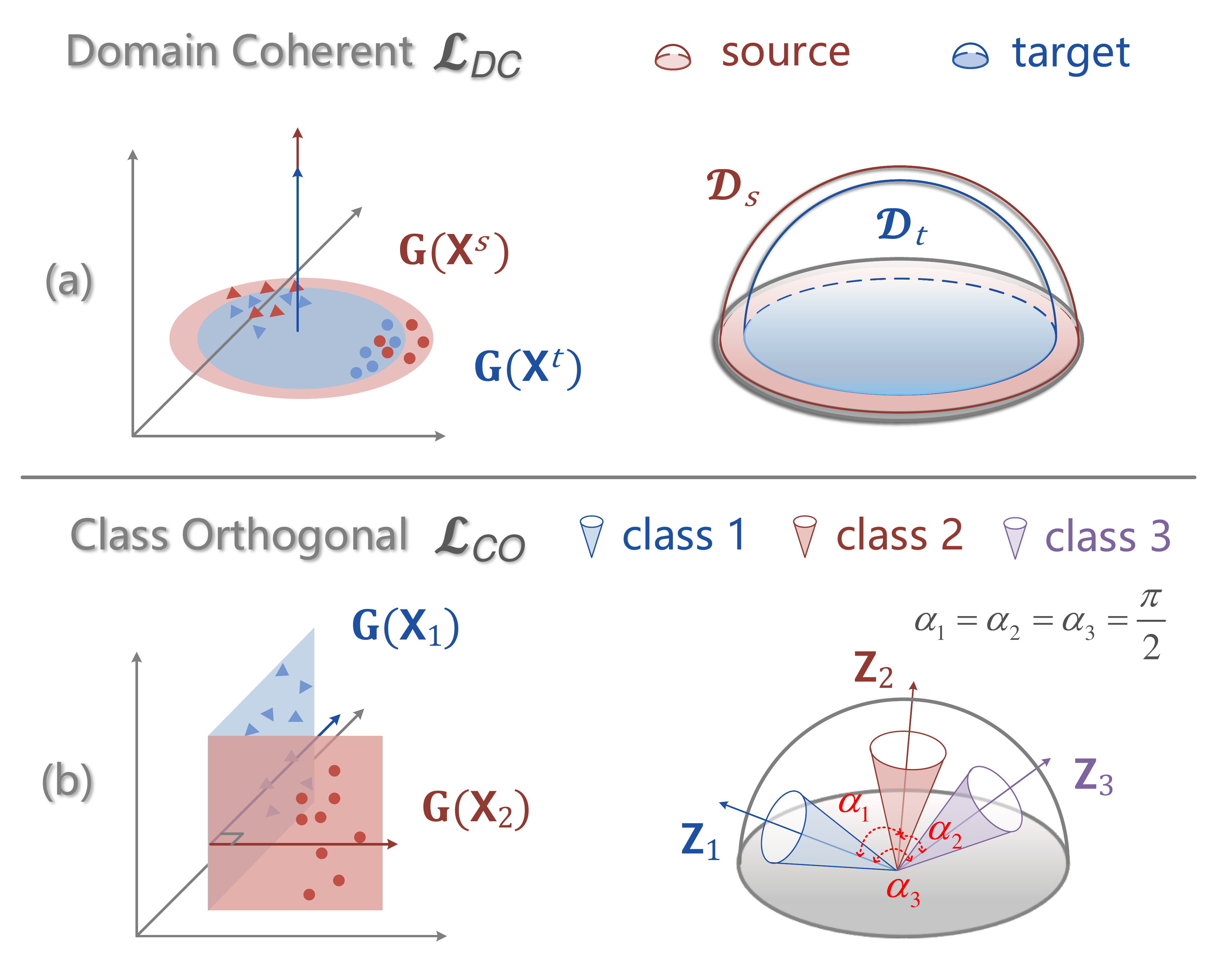} 
    \caption{Illustration of GET. \textbf{(a)}: The left figure depicts that domains are overlapped to maximize the coherence; the right figure shows the dimensions and coherence of subspaces are maximized when $\mathcal{L}_{DC}$ is optimized. \textbf{(b)}: The orthogonality between different classes enhances the discriminability. Under GET, the subspaces of different classes are orthogonal which maximizes the angles between hyperplanes. }
    \label{fig:method_illustration}
    \end{center}
    \vskip -0.2in
\end{figure}

\subsection{Geometry-Aware Transfer}
In this section, we first introduce the domain coherence and class orthogonality for building GET, and connect them with the transferability and discriminability, respectively. Theoretical results are derived to provide a foundation for norm-based learning in UDA. Then we propose the geometry-aware constraint which shows the possibility of enhancing transferability and discriminability simultaneously.

\subsubsection{Domain Coherent Constraint}
We first consider the transferability. Let $\mathbf{G}(\cdot):\mathbb{R}^{D}\rightarrow \mathbb{R}^{d}$ be feature projection. To learn representations with the desired geometry property, we first map the raw input $\mathbf{x}$ into the $d$-dimensional subspace as $\mathbf{z}=\mathbf{G}(\mathbf{x})$. Note that the notations for $\mathbf{Z}$ are similar to $\mathbf{X}$. As Figure~\ref{fig:method_illustration}(a), the domain coherence aims to learn a subspace that domains are overlapped. Thus, we define the \textit{domain coherent} constraint in rank form as
\begin{small}
\begin{equation*}\label{eq:domain_coherent_rank}
    \mathop{\arg \max}_{\textbf{G}} ~ \mathcal{J}_{DC} = \rank (\mathbf{Z}^s) + \rank (\mathbf{Z}^t) - \rank (\mathbf{Z}).
\end{equation*}
\end{small}\noindent
It is clear that $\mathcal{J}_{DC}\leq \min\{\rank (\mathbf{Z}^s) , \rank (\mathbf{Z}^t)\}$ from Eq. \eqref{eq:rank-lower-bound}. So the maximum of above objective is reached when the subspaces of $\mathbf{Z}^s$ and $\mathbf{Z}^t$ are overlapped, \textit{i.e.}, the representations of domains are maximum correlated.

However, as discussed before, the optimization of objective $\mathcal{J}_{DC}$ is non-convex and NP-hard. We reformulate the \textit{domain coherent} constraint in nuclear norm form as
\begin{small}
\begin{equation}\label{eq:domain_coherent_nuclear_global}
    \mathop{\arg \max}_{\textbf{G}} ~ \mathcal{L}_{DC} = \| \mathbf{Z}^s \|_* + \| \mathbf{Z}^t \|_* - \| \mathbf{Z} \|_*.
\end{equation}
\end{small}\noindent
We are interested in: 1) the upper bound of $\mathcal{L}_{DC}$, which is equivalent to the lower bound inequality Eq. \eqref{eq:rank-lower-bound} in nuclear norm form; 2) the condition for achieving the supremum of $\mathcal{L}_{DC}$, \textit{i.e.}, the condition for equality. These problems are essential for the interpretability and optimization of $\mathcal{L}_{DC}$.

For simplicity, we denote $\mathbf{A}\in \mathbb{R}^{d\times n}$, $\mathbf{B}\in \mathbb{R}^{d\times m}$ where row and column can be taken as the feature dimension and sample-size. As we consider the model in low-dimensional projective space, it is reasonable to assume $d\leq \min\{n,m\}$. We next show the results for above problem.
\begin{theorem}[Transferability]\label{thm:domain_coherent}
Assuming that $\| \mathbf{A} \|_\sigma \leq \alpha$, $\| \mathbf{B} \|_\sigma \leq \alpha$, where $\| \mathbf{A} \|_\sigma$ is the spectral norm, then \\
\textbf{(\romannumeral1)} $\|\mathbf{A}\|_*+\|\mathbf{B}\|_* - \|[\mathbf{A},\mathbf{B}]\|_* \leq (2-\sqrt{2})\alpha d$; \\
\textbf{(\romannumeral2)} the equality in (\romannumeral1) holds when $\mathbf{A}$, $\mathbf{B}$ have the same column spaces and $\|\mathbf{A}\|_*=\|\mathbf{B}\|_*=\alpha d$.
\end{theorem}
\begin{proof}
    \textbf{(\romannumeral1)} Denote $\mathbf{C}=[\mathbf{A},\mathbf{B}]$. Note that matrix is a second-order tensor, and $\{\mathbf{A},\mathbf{B}\}$ is a 2-regular partition of $\mathbf{C}$, then the following inequality holds \cite[Theorem 3.1]{li2016bounds}:
    \begin{small}
    \begin{equation*}\label{eq:tensor_partition_ineq}
       \left\| \big[\|\mathbf{A}\|_*,\|\mathbf{B}\|_*\big] \right\|_2 \leq
       \| \mathbf{C} \|_*.
    \end{equation*}
    \end{small}\noindent
    By applying it into $\|\mathbf{A}\|_*+\|\mathbf{B}\|_* - \|\mathbf{C}\|_*$, we have
    \begin{small}
    \begin{equation*}\label{eq:nuclear_bound}
        \|\mathbf{A}\|_*+\|\mathbf{B}\|_* - \|\mathbf{C}\|_*
        \leq
        \|\mathbf{A}\|_*+\|\mathbf{B}\|_* - \left\| [\|\mathbf{A}\|_*,\|\mathbf{B}\|_*] \right\|_2.
    \end{equation*}
    \end{small}\noindent
    Let $\|\mathbf{A}\|_*=x$, $\|\mathbf{B}\|_*=y$ and $f(x,y)=x+y-\sqrt{x^2+y^2}$, the above inequality is written as $\|\mathbf{A}\|_*+\|\mathbf{B}\|_* - \|\mathbf{C}\|_* \leq f(x,y)$. Note $x\in[0,\alpha d]$ also holds since $\| \mathbf{A} \|_\sigma \leq \alpha$ ($y$ is similar). Now we show that $f(x,y)$ is monotone \textit{w.r.t.} a partial order on $\mathbb{R}^2$. Let $\mathbf{z}=[x,y]^T$, $\Delta \mathbf{z}=[\Delta x,\Delta y]^T$, then $\forall \Delta x, \Delta y > 0$, we have
    \begin{small}
    \begin{eqnarray}
        &~ & f(x+\Delta x,y+ \Delta y )- f(x,y)  \nonumber \\
        &=& \| \Delta \mathbf{z} \|_1 - \left(\| \mathbf{z} + \Delta \mathbf{z} \|_2 - \| \mathbf{z} \|_2 \right)  \nonumber \\
        &\geq & \| \Delta \mathbf{z} \|_1 - \| \Delta \mathbf{z} \|_2  \label{eq:minkowski_ineq} \\
        &\geq & 0, \nonumber \label{eq:1_2_norm_ineq}
    \end{eqnarray}
    \end{small}\noindent
    where Eq.~\eqref{eq:minkowski_ineq} holds from Minkowski's inequality.

    As $\mathbf{z}\in [0,\alpha d]^2$ and $f(x,y)$ is monotone, the maximum is attained at boundary point, \textit{i.e.}, $f(\alpha d,\alpha d)=(2-\sqrt{2})\alpha d$. Finally, $\|\mathbf{A}\|_*+\|\mathbf{B}\|_* - \|\mathbf{C}\|_* \leq \sup f(x,y) \leq (2-\sqrt{2})\alpha d $.

    \textbf{(\romannumeral2)} As $\mathbf{A}$, $\mathbf{B}$ have the same column spaces, the Singular Value Decomposition (SVD) of $\mathbf{A}$ and $\mathbf{B}$ can be written as $\mathbf{A}=\mathbf{U}\mathbf{\Sigma}\mathbf{V}_1^T$ and $\mathbf{B}=\mathbf{U}\mathbf{\Lambda}\mathbf{V}_2^T$. Then the nuclear norm of $\mathbf{C}$ is $\sum_{i=1}^d \sqrt{\sigma_i^2+\lambda_i^2}$ since $\| \mathbf{C} \|_*=\text{tr}(\sqrt{\mathbf{C}\mathbf{C}^T})$ and
    \begin{small}
    \begin{equation*}
        \mathbf{C}\mathbf{C}^T=\mathbf{A}\mathbf{A}^T+\mathbf{B}\mathbf{B}^T
        = \mathbf{U} (\mathbf{\Sigma}^2 + \mathbf{\Lambda}^2) \mathbf{U}^T.
    \end{equation*}
    \end{small}\noindent
    Recall that $\|\mathbf{A}\|_*=\alpha d$ and $\| \mathbf{A} \|_\sigma \leq \alpha$, then $\forall i,~\sigma_i=\alpha$ ($\lambda_i$ is similar). Finally, we have $\|\mathbf{A}\|_*+\|\mathbf{B}\|_* - \|\mathbf{C}\|_*
    = \sum_{i=1}^d \sigma_i + \lambda_i - \sqrt{\sigma_i^2+\lambda_i^2}
    =  (2-\sqrt{2})\alpha d.$
\end{proof}
Theorem \ref{thm:domain_coherent}\textbf{(\textit{\romannumeral1)}} shows that the domain coherent objective $\mathcal{L}_{DC}$ is upper-bounded by $(2-\sqrt{2})\alpha d$ in a ball of finite radius $\alpha$. Since nuclear norm is the convex envelop of rank over the unit ball \cite{fazel2001rank}, $\alpha$ is usually set as 1 in application. This guarantees that the optimization problem $\mathcal{L}_{DC}$ is well-defined and have non-trivial solution.

Theorem \ref{thm:domain_coherent}\textbf{(\textit{\romannumeral2)}} ensures the interpretability of GET model for learning transferable representations, and provides an insight into the norm-based learning. Specifically, $\mathcal{L}_{DC}$ is maximized when the source and target domains share the same subspace basis. Besides, the norms of representations are required to be maximized as $\|\mathbf{Z}^s\|_*=\|\mathbf{Z}^t\|_*=\alpha d$. It implies that the larger norm is sufficient to learn two linearly dependent domain subspaces with overlapping subspace bases. This result provides a theoretical justification to recent results in learning transferability \cite{xu2019larger,cui2020towards} which empirically observe that the representations with larger norms are more preferable. In fact, the overlap of subspace bases ensures the geometric interpretability for $\mathcal{L}_{DC}$, which means the domains can be expressed by each other as the linear combinations of basis.

As suggested in recent results \cite{yan2017mind,zhao2019learning}, the global alignment may distort the intrinsic structure of clusters. To preserve the intrinsic structure while learning domain coherent representations, we further reformulate the objective $\mathcal{L}_{DC}$ as a class-level learning problem.
\begin{small}
\begin{equation}\label{eq:domain_coherent_nuclear}
    \mathop{\arg \max}_{\textbf{G}} ~ \mathcal{L}_{DC} =\sum_{i=1}^k \| \mathbf{Z}_i^s \|_* + \| \mathbf{Z}_i^t \|_* - \| \mathbf{Z}_i \|_*.
\end{equation}
\end{small}\noindent

The fact that two subspaces are overlapped means the principal angles between them are 0. Therefore, the optimization of domain coherence $\mathcal{L}_{DC}$ will project the clusters of different domains into the same subspaces as shown in Figure \ref{fig:method_illustration}(a), which enhances the class-level transferability.

\subsubsection{Class Orthogonal Constraint}
As mentioned above, the smaller the rank, the higher the relevance. To achieve a better recognition performance with good geometric interpretability, two things are necessary. On the one hand, the relevance of features within the same class should be high, that is, the rank is small. Thus, the ranks of clusters should be as small as possible to learn a compact representation space. On the other hand, features of different classes should be as unrelated as possible, which means the overall rank should be large to ensure the diversity between clusters. Thus, they are as dissimilar to each other as possible. Based on these points, we can propose the \textit{class orthogonal} constraint as

\begin{small}
\begin{equation*}
    \mathop{\arg \min}_{\textbf{G}} ~ \mathcal{J}_{CO} =   \sum_{i=1}^k \rank (\mathbf{Z}_i) - \rank (\mathbf{Z}).
\end{equation*}
\end{small}\noindent

The above formulation mainly have two drawbacks: 1) the optimization problem of rank-based objective; 2) the rank-based objective can only ensure the linear independence but not orthogonality. For these reasons, we similarly propose the \textit{class orthogonal} constraint in nuclear norm form as
\begin{small}
\begin{equation}
\mathop{\arg \min}_{\textbf{G}} ~ \mathcal{L}_{CO} =   \sum_{i=1}^k \|\mathbf{Z}_i\|_*- \|\mathbf{Z}\|_*.
\end{equation}
\end{small}\noindent
According to Theorem \ref{thm:class_orthogonal}, the infimum of $\mathcal{L}_{CO}$ is reached when $\mathbf{Z}_i$ are orthogonal as shown in Figure \ref{fig:method_illustration}(b). In this case, the subspaces of different clusters will try to be orthogonal to each other and the principal angles will be maximized.

\subsubsection{Geometry-Aware Constraint}
Now we focus on the possibility of enhancing transferability and discriminability simultaneously. We first define \textit{geometry-aware} constraint as
\begin{small}
\begin{equation}
\mathop{\arg \min}_{\textbf{G}} ~ \mathcal{L}_{GA} = \lambda_{CO}\mathcal{L}_{CO} - \lambda_{DC}\mathcal{L}_{DC},
\end{equation}
\end{small}\noindent
where $\lambda_{DC}$ and $\lambda_{CO}$ are parameters.

Note that the choice of $\lambda_{DC}$ and $\lambda_{CO}$ can be taken as the trade-off between transferability and discriminability. Specifically, $\mathcal{L}_{CO}$ tries to learn compact subspaces for clusters but $\mathcal{L}_{DC}$ tends to learn a ``loose'' space with larger norm. Denote $\lambda = \frac{\lambda_{DC}}{\lambda_{CO}}$. If $\lambda \rightarrow \infty$, the minimum of $\mathcal{L}_{GA}$ is dominated by the maximum of domain coherence $\mathcal{L}_{DC}$ which is $(2-\sqrt{2})\alpha k d$. In this case, all clusters share a same subspace and discriminability is lost. In contrast, if $\lambda \rightarrow 0$, the minimum of $\mathcal{L}_{GA}$ is dominated by $\mathcal{L}_{CO}$ which is 0. The model may learn a trivial solution such that $\mathbf{Z}=\mathbf{0}$ and $\mathcal{L}_{GA}\rightarrow 0$.

Hopefully, such a trade-off problem can be addressed by choosing parameters carefully. Then $\mathcal{L}_{GA}$ and $\mathcal{L}_{CO}$ will be the regularization for each other. We next show how to balance the parameters to achieve a favorable geometric structure via \textit{geometry-aware} constraint $\mathcal{L}_{GA}$.

\begin{theorem}[Trade-off]\label{thm:geometry_aware}
Denote $\tilde{\mathcal{L}}_{GA}= \frac{\mathcal{L}_{GA}}{\lambda_{CO}}$, $\lambda = \frac{\lambda_{DC}}{\lambda_{CO}}$. Assuming that $\| \mathbf{Z}^{s/t} \|_\sigma \leq \alpha$ and $\| \mathbf{Z} \|_\sigma \leq \sqrt{2}\alpha$.  \\
\textbf{(\romannumeral1)} If $\lambda \geq 1+ \sqrt{2}$, then
\begin{equation*}
\tilde{\mathcal{L}}_{GA} \geq \left[\left((\sqrt{2}-2)\lambda + \sqrt{2}\right)\sqrt{k}-\sqrt{2}\right] \alpha d ,
\end{equation*}
with equality when $\mathbf{Z}^s_i$, $\mathbf{Z}^t_i$ have the same column spaces and $\| \mathbf{Z}^{s/t}_i \|_* = \frac{1}{\sqrt{k}}\alpha d$ ($\mathbf{Z}_1, \mathbf{Z}_2, \ldots, \mathbf{Z}_k$ are not orthogonal); \\
\textbf{(\romannumeral2)} If $\lambda \leq 1+ \sqrt{2}$, then
\begin{equation*}
\tilde{\mathcal{L}}_{GA} \geq (\sqrt{2}-2)\alpha \lambda d ,
\end{equation*}
with equality when $\mathbf{Z}^s_i$, $\mathbf{Z}^t_i$ have the same column spaces and the column spaces of $\mathbf{Z}_1, \mathbf{Z}_2, \ldots, \mathbf{Z}_k$ are orthogonal.
\end{theorem}

Theorem \ref{thm:geometry_aware} splits $\tilde{\mathcal{L}}_{GA}$ into two cases. In case \textbf{(\romannumeral1)} with $\lambda \geq 1+ \sqrt{2}$, the class-wise domain coherence $\mathcal{L}_{DC}$ is over-learned where orthogonality property of clusters is lost. This is because $\| \mathbf{Z}^{s/t}_i \|_* = \frac{1}{\sqrt{k}}\alpha d\Rightarrow \mathcal{L}_{DC}>0$, then the orthogonality condition will never be reached. That is to say, the enhancement of transferability will lead to a decreasing discriminability when $\lambda \geq 1+ \sqrt{2}$. Case \textbf{(\romannumeral2)} suggests that the proper parameters should satisfy that $\lambda_{DC}/\lambda_{CO}\leq 1+ \sqrt{2}$. In this case, class orthogonality $\mathcal{L}_{CO}$ provides a regularization that prevents the class-wise domain coherence $\mathcal{L}_{DC}$ from learning overlarge subspaces for clusters. Note that $\lambda$ should be greater than 0, since $\mathcal{L}_{DC}$ also provides a regularization for $\mathcal{L}_{CO}$ where the trivial solution $\mathbf{0}$ will be avoided. Besides, Theorem \ref{thm:geometry_aware} \textbf{(\romannumeral2)} implies that it is possible to learn the transferability and discriminability simultaneously, which demonstrates that the domain coherence and class orthogonality can benefit from each other. We will verify this conclusion in the hyper-parameters experiment which shows that the parameters should be balanced (\textit{i.e.}, $\lambda\approx 1$).

In summary, the goal of geometry-aware constraint is to make the representations of the same class as compact as possible, and the subspaces of different classes orthogonal. Besides, the subspaces of different domains are enlarged as much as possible to learn the domain coherence, which is also compatible with the learning of diversity clusters, \textit{i.e.}, maximum norm with full-rank.

\subsubsection{Overall Model}
The overall GET model consists of two parts: classification objective and geometry-aware constraint. Denote the classifier as $\mathbf{C}(\cdot)$ and its probabilistic predictions by $\hat{\mathbf{Y}}=\mathbf{C}(\mathbf{Z})$, we will optimize the feature projection $\mathbf{G}$ and classifier $\mathbf{C}$, which is instantiated by neural networks, in an end-to-end manner.

A common way to train a basic model for classification task is empirical risk minimization. Here we apply this rule with entropy-based cost functions to both the domains. First, the cross-entropy objective is applied to the source domain with ground-truth label for supervised learning:
\begin{small}
\begin{equation*}
\mathop{\arg \min}_{\textbf{G}, ~\textbf{C}}~ \mathcal{L}^s_{E} = \sum_{j=1}^{n^s}\sum_{i=1}^k - \mathbf{Y}_{ij}^s \log \hat{\mathbf{Y}}_{ij}^s.
\end{equation*}
\end{small}\noindent

For the target domain without labels, the entropy is applied for unsupervised learning, which also can be considered as a regularization. It preserves the classification knowledge and minimizes the uncertainty of prediction.
\begin{small}
\begin{equation*}
 \mathop{\arg \min}_{\textbf{G}, ~\textbf{C}}~ \mathcal{L}^t_{E} = \sum_{j=1}^{n^t}\sum_{i=1}^k - \hat{\mathbf{Y}}_{ij}^t \log \hat{\mathbf{Y}}_{ij}^t.
\end{equation*}
\end{small}\noindent

By integrating the geometry-aware constraint, the overall objective of GET model can be written as:
\begin{small}
\begin{equation}\label{eq:overall_objective}
\mathop{\arg \min}_{\textbf{G}, ~\textbf{C}}~ \mathcal{L} = \mathcal{L}_{Cls}(\textbf{G},\textbf{C}) + \mathcal{L}_{GA}(\textbf{G}),
\end{equation}
\end{small}\noindent
where $\mathcal{L}_{Cls}=\mathcal{L}^s_{E} + \lambda_t \mathcal{L}^t_{E}$ is the classification objective.

\begin{algorithm}[t]
    \caption{GET for UDA}
    \label{alg:GETforUDA}
    \begin{algorithmic}[1] 
    \REQUIRE Source dataset $\mathcal{D}_s$, Target dataset $\mathcal{D}_t$, Warming up epochs $T_{warm}$, GET epochs $T_{adapt}$, Learning rate $\lambda$;\\
    \ENSURE Feature projection $\mathbf{G}(\cdot)$, Classifier $\mathbf{C}(\cdot)$;
    \STATE Initialize the network parameters $\Theta=\{\Theta_G,\Theta_C\}$; \\
    \% \textit{Warming Up Stage}
    \FOR {$iter=1,2,...,T_{warm}$}
    \STATE Forward propagate $\{x_i^s\}_{i=1}^{n^s}$ according to $\mathbf{Z}=\mathbf{G}(\mathbf{X})$ and $\hat{\mathbf{Y}}=\mathbf{C}(\mathbf{Z})$; \\
    \STATE Compute $\mathcal{L}_{warm}=\mathcal{L}^s_{E}+\mathcal{L}_{DC}+\mathcal{L}_{CO}$, where $\mathcal{L}_{DC}$ follows Eq.~\eqref{eq:domain_coherent_nuclear_global} and $\mathcal{L}_{CO}$ only uses source data;
    \STATE Update: $\Theta\leftarrow\Theta-\lambda\triangledown\mathcal{L}_{warm}(\Theta)$;
    \ENDFOR\\
    \% \textit{GET Learning Stage}
    \FOR {$iter=1,2,...,T_{adapt}$}
    \STATE Forward propagate $\{x_i^s\}_{i=1}^{n^s}$ and $\{x_i^t\}_{i=1}^{n^t}$ and compute the pseudo-labels $\bar{\mathbf{Y}}^t$;
    \STATE Compute the overall GET objective $\mathcal{L}$ as Eq.~\eqref{eq:overall_objective};
    \STATE Update: $\Theta\leftarrow\Theta-\lambda\triangledown\mathcal{L}(\Theta)$;
    \ENDFOR
    \end{algorithmic}
\end{algorithm}

\begin{table*}[t]
\centering
\caption{Accuracies (\%) on \textbf{Office-Home} (ResNet-50)}
\label{tab:compare_home}
\renewcommand{\tabcolsep}{0.25pc} 
\vskip 0.05in
\begin{center}
\begin{small}
\begin{tabular}{cccccccccccccc}
\toprule
Method & \multicolumn{1}{l}{Ar$\rightarrow$Cl} & \multicolumn{1}{l}{Ar$\rightarrow$Pr} & \multicolumn{1}{l}{Ar$\rightarrow$Rw} & \multicolumn{1}{l}{Cl$\rightarrow$Ar} & \multicolumn{1}{l}{Cl$\rightarrow$Pr} & \multicolumn{1}{l}{Cl$\rightarrow$Rw} & \multicolumn{1}{l}{Pr$\rightarrow$Ar} & \multicolumn{1}{l}{Pr$\rightarrow$Cl} & \multicolumn{1}{l}{Pr$\rightarrow$Rw} & \multicolumn{1}{l}{Rw$\rightarrow$Ar} & \multicolumn{1}{l}{Rw$\rightarrow$Cl} & \multicolumn{1}{l}{Rw$\rightarrow$Pr} & \multicolumn{1}{l}{Avg.} \\ \midrule
Source               & 34.9                      & 50.0                      & 58.0                      & 37.4                      & 41.9                      & 46.2                      & 38.5                      & 31.2                      & 60.4                      & 53.9                      & 41.2                      & 59.9                      & 46.1                     \\
DAN                  & 43.6                      & 57.0                      & 67.9                      & 45.8                      & 56.5                      & 60.4                      & 44.0                      & 43.6                      & 67.7                      & 63.1                      & 51.5                      & 74.3                      & 56.3                     \\
DANN                 & 45.6                      & 59.3                      & 70.1                      & 47.0                      & 58.5                      & 60.9                      & 46.1                      & 43.7                      & 68.5                      & 63.2                      & 51.8                      & 76.8                      & 57.6                     \\
SAFN                   & 52.0                      & 71.7                      & 76.3            & 64.2                     &69.9                    & 71.9                      & 63.7                      & 51.4                      & 77.1                    & 70.9                      & 57.1                     & 81.5                     & 67.3                     \\
ETD                   & 51.3                      & 71.9                      & \textbf{85.7}             & 57.6                      & 69.2                      & 73.7                      & 57.8                      & 51.2                      & 79.3                      & 70.2                      & 57.5                      & 82.1                      & 67.3                     \\
ALDA   & \textbf{53.7} & 70.1 & 76.4 & 60.2 & 72.6 &71.5 &56.8 & 51.9 &77.1 &70.2&56.3 &82.1 &66.6\\
DMP                   & 52.3                      & 73.0                      & 77.3                      & \textbf{64.3}             & 72.0                      & 71.8                      & 63.6                      & 52.7                      & 78.5                      & \textbf{72.0}                      & \textbf{57.7}                      & 81.6                      & 68.1                     \\
BNM & 52.3 & 73.9 & 80.0 & 63.3 & 72.9 & \textbf{74.9} & 61.7 & 49.5 & 79.7 & 70.5 & 53.6 & 82.2 & 67.9 \\
\midrule
GET                 & 52.1                      & \textbf{75.6}             & 77.4                      & 61.6                      & \textbf{74.6}             & 73.7             & \textbf{65.0}             & \textbf{53.1}             & \textbf{80.8}             & 69.5                      & 55.9                      & \textbf{83.6}             & \textbf{68.6}            \\
\bottomrule
\end{tabular}
\end{small}
\end{center}
\vskip -0.1in
\end{table*}

\subsubsection{Optimization}
Since the geometry-aware constraint $\mathcal{L}_{GA}$ requires the labels of the target data, we use the classifier to assign pseudo-labels $\bar{\mathbf{Y}}^t\in\mathbb{R}^{k\times n^t}$ for target domain. Specifically, $\bar{\mathbf{Y}}_{ij}^t = 1$ if $i=\mathop{\arg \max}_l \hat{\mathbf{Y}}_{lj}^t$. To ensure the accuracy of subsequent tests, we set a threshold $\tau$ to select pseudo-label. The pseudo-label will be used only when the prediction probability is greater than threshold $\tau$.

To reduce the uncertainty of pseudo-label, we design a multi-stage pipeline. As Algorithm \ref{alg:GETforUDA}, there are two components: warming up and geometry structure learning. During the first stage, the source data with ground truth labels and target data without labels are used. At this point, the source classification objective $\mathcal{L}^s_{E}$, class orthogonal objective $\mathcal{L}_{CO}$ and global domain coherence $\mathcal{L}_{DC}$ in Eq.~\eqref{eq:domain_coherent_nuclear_global} are optimized, where only source data are applied to $\mathcal{L}_{CO}$. During the geometry learning stage, the entire GET model is optimized according to Eq. \eqref{eq:overall_objective}, where the source data with ground truth labels and target data with pseudo-labels are used.

\section{Experiments}
\noindent\textbf{Office-Home}~\cite{venkateswara2017deep} contains object images to form four domains: Artistic (Ar), Clipart (Cl), Product (Pr) and Real-World (Rw). Each domain contains 65 object categories, and they amount to around 15,500 images. 

\noindent\textbf{ImageCLEF}~\cite{caputo2014imageclef} consists of three domains: Caltech-256 (C), ImageNet (I) and Pascal-VOC (P). There are 12 categories and each class contains 50 images. 

\noindent\textbf{Office31}~\cite{saenko2010adapting} contains 4,110 images of 31 categories collected from three various domains: Amazon (A), Web camera (W) and Digital SLR camera (D). 
\subsection{Set Up}

We employ ResNet-50~\cite{he2016deep} as the backbone network. The network weights are optimized by ADAM with 0.01 weight decay. The learning rate is chosen from $[10^{-4}, 10^{-3}]$. On ImageCLEF, Office31 and Office-Home, we respectively train GET for 200, 300 and 350 epochs, where $T_{warm}=20$. The hyper-parameters are selected from $\{1,5,10\}\times10^{-4}$. For each transfer task, we report the average classification accuracy over ten random repeats.

\subsection{Comparison}

\begin{table}[t]
\caption{Accuracies (\%) on \textbf{ImageCLEF (top)} and \textbf{Office31 (bottom)} (ResNet-50).}
\label{tab:compare_clef&31}
\renewcommand{\tabcolsep}{0.4pc} 
\begin{center}
\begin{small}
\begin{tabular}{cccccccc}
\toprule
Method & I$\rightarrow$P           & P$\rightarrow$I           & I$\rightarrow$C           & \multicolumn{1}{c}{C$\rightarrow$I}           & \multicolumn{1}{c}{C$\rightarrow$P}           & \multicolumn{1}{c}{P$\rightarrow$C}  & Avg.          \\
\midrule
Source                 & 74.8      & 83.9      & 91.5      & 78.0                          & 65.5                          & 91.2                 & 80.7          \\
DAN                   & 74.5      & 82.2      & 92.8      & 86.3                          & 69.2                         & 89.8                 & 82.5          \\
DANN                  & 75.0      & 86.0      & 96.2      & 87.0                          & 74.3                          & 91.5                 & 85.0          \\
SAFN  & 79.3 &93.3 &96.3&91.7&77.6&95.3&88.9\\
ETD                   & 81.0          & 91.7          & 97.9 & \multicolumn{1}{c}{93.3}          & \multicolumn{1}{c}{79.5}          & \multicolumn{1}{c}{95.0} & 89.7          \\
DMP                  & 80.7      & 92.5      & 97.2      & 90.5                          & 77.7                          & 96.2                 & 89.1          \\
DWL        & \textbf{82.3}     & \textbf{94.8}      &\textbf{98.1}      & 92.8                          & 77.9                          & \textbf{97.2}        & 90.5          \\
\midrule 
GET                   & 82.2 & 94.1 & 97.3          & \textbf{95.6} & \textbf{82.3} & 96.4 & \textbf{91.3} \\
\bottomrule
\end{tabular}
\\[4pt]
\renewcommand{\tabcolsep}{0.245pc} 
\begin{tabular}{cccccccc}
\toprule
Method & A$\rightarrow$W           & A$\rightarrow$D           & W$\rightarrow$A           & W$\rightarrow$D          & D$\rightarrow$A           & D$\rightarrow$W          & Avg.          \\
\midrule
Source & 68.4 & 68.9 & 60.7 & 99.3 & 62.5 & 96.7 & 76.1 \\
DAN                & 80.5          & 78.6          & 62.8          & 99.6         & 63.6          & 97.1         & 80.4          \\
DANN              & 82.0          & 79.7          & 67.4          & 96.9         & 68.2          & 96.9         & 82.2          \\
SAFN & 90.3 & 92.1 &  71.2 &  100.0 & 73.4 & 98.7 &  87.6 \\
ETD                & 92.1          & 88.0          & 67.8          & 100.0          & 71.0          & \textbf{100.0} & 86.2          \\
DMP              & 93.0          & 91.0          & 70.2          & 100.0          & 71.4          & 99.0         & 87.4          \\
BNM & 91.5 & 90.3 & 71.6 & 100.0 & 70.9 & 98.5 & 87.1 \\
DWL               & 89.2          & 91.2 & 69.8         & 100.0          & 73.1         & 99.2         & 87.1          \\
\midrule 
GET               & \textbf{93.9} & \textbf{93.1}          & \textbf{73.0} & \textbf{100.0} & \textbf{78.0} & 98.7         & \textbf{89.5} \\
\bottomrule
\end{tabular}
\end{small}
\end{center}
\vskip -0.1in
\end{table}

\begin{figure*}[t]
    \centering
    \subfigure[ImageCLEF I$\rightarrow$P \label{fig:3D_clefIP}]{\includegraphics[width=0.245\linewidth,height=77pt,trim=120 60 100 70,clip]{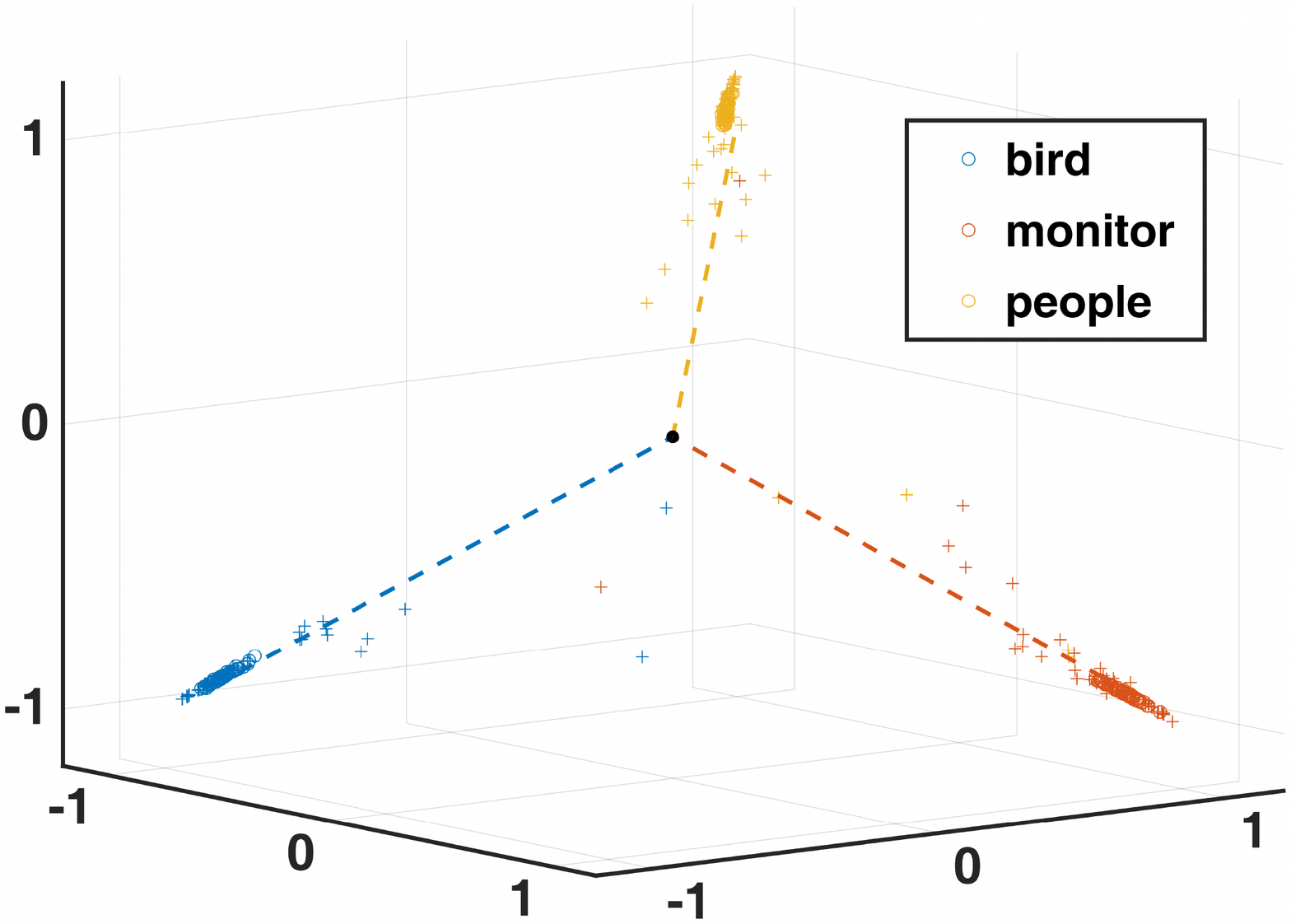}}
    \hfill
    \subfigure[Office31 A$\rightarrow$W \label{fig:3D_31AW}]{\includegraphics[width=0.245\linewidth,height=77pt,trim=120 60 100 70,clip]{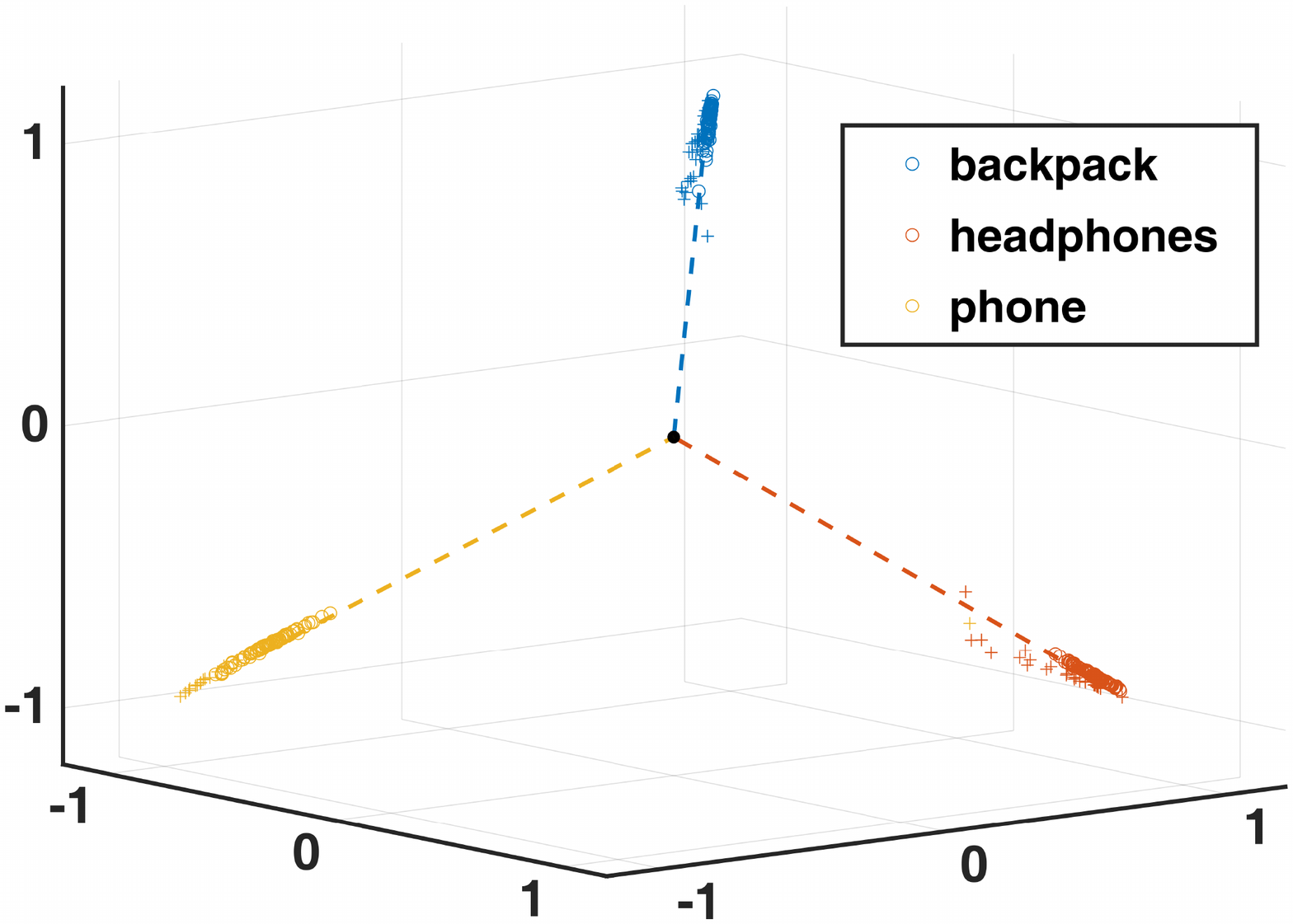}}
    \hfill
    \subfigure[ImageCLEF I$\rightarrow$P \label{fig:hyper_clef}]{\includegraphics[width=0.245\linewidth,height=77pt,trim=90 60 50 70,clip]{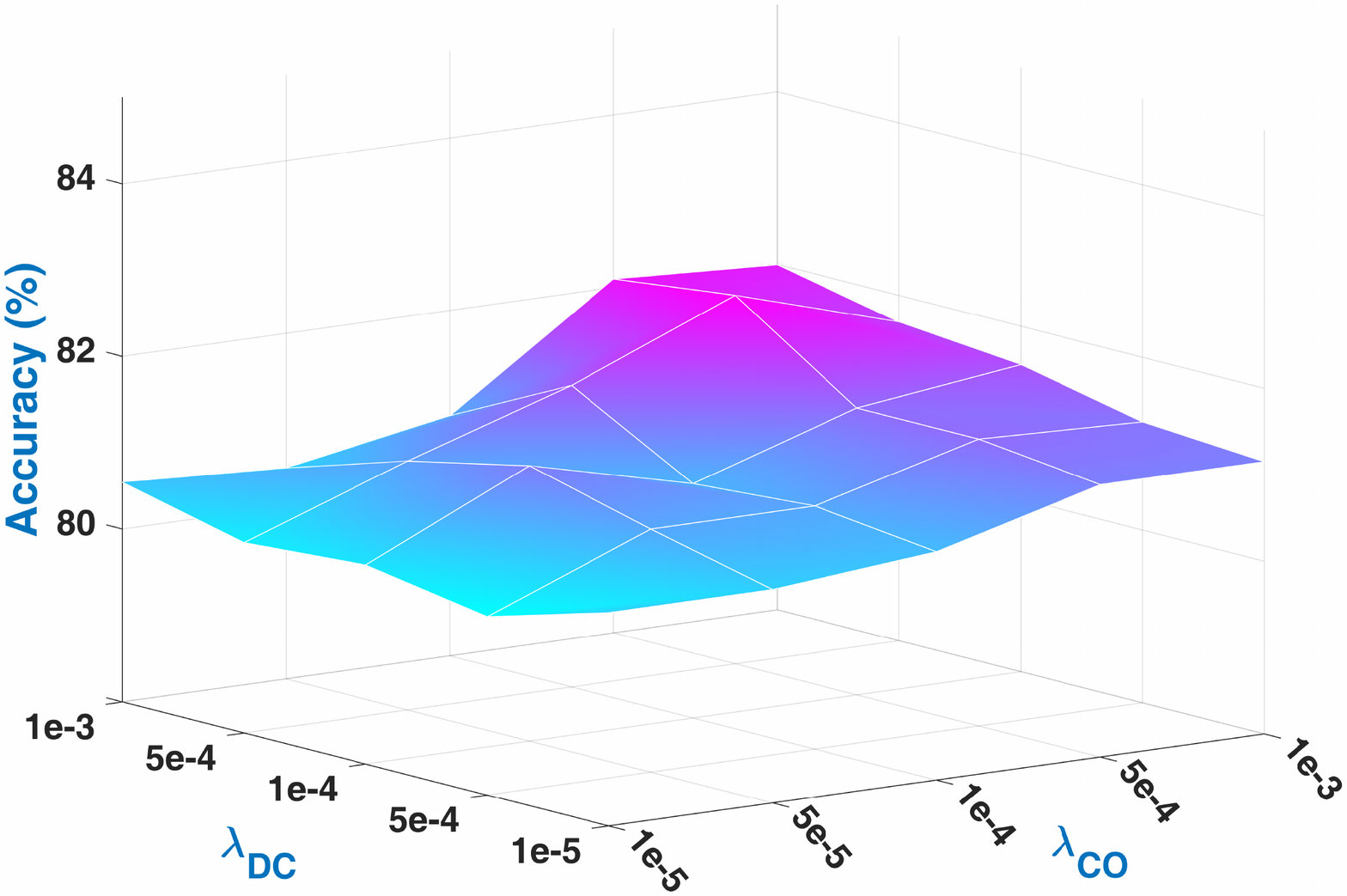}}
    \hfill
    \subfigure[Office31 A$\rightarrow$W \label{fig:hyper_31}]{\includegraphics[width=0.245\linewidth,height=77pt,trim=90 60 50 70,clip]{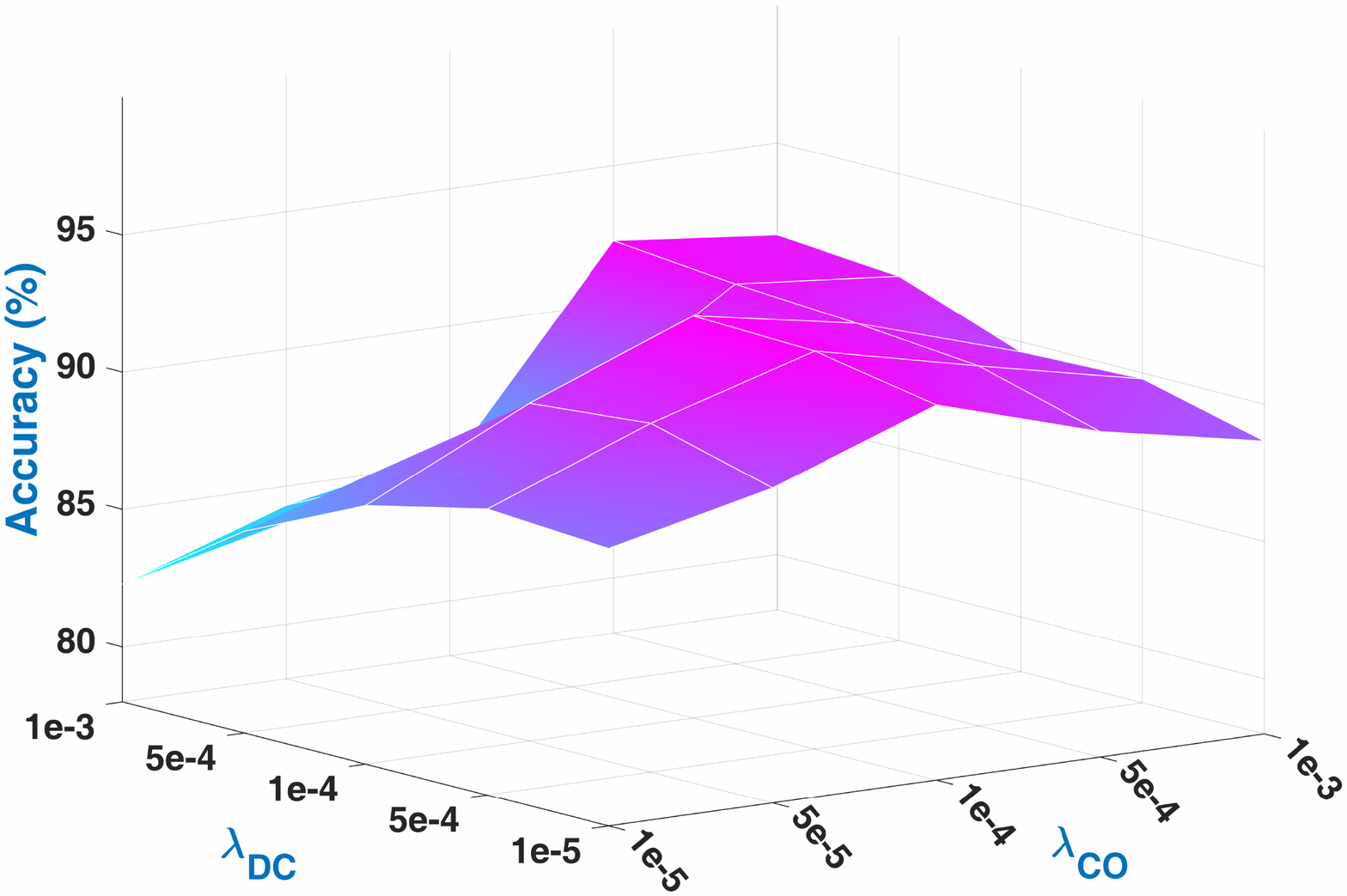}} \\
    \subfigure[ResNet (ImageCLEF) \label{fig:tsne_clef_bf_class}]{\includegraphics[width=0.245\linewidth,height=77pt,trim=190 115 180 100,clip]{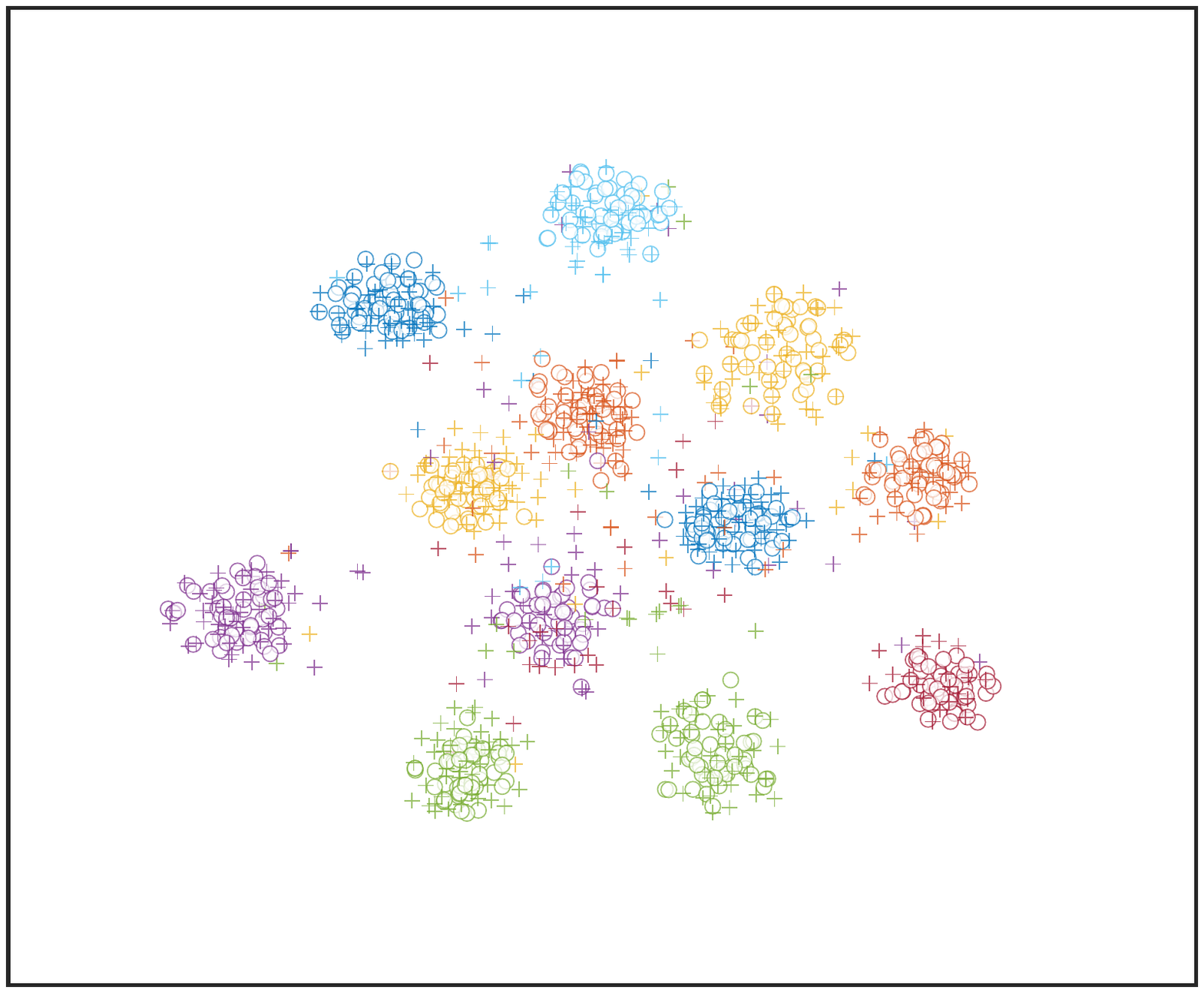}}
    \hfill
    \subfigure[GET (ImageCLEF) \label{fig:tsne_clef_ad_class}]{\includegraphics[width=0.245\linewidth,height=77pt,trim=190 115 180 100,clip]{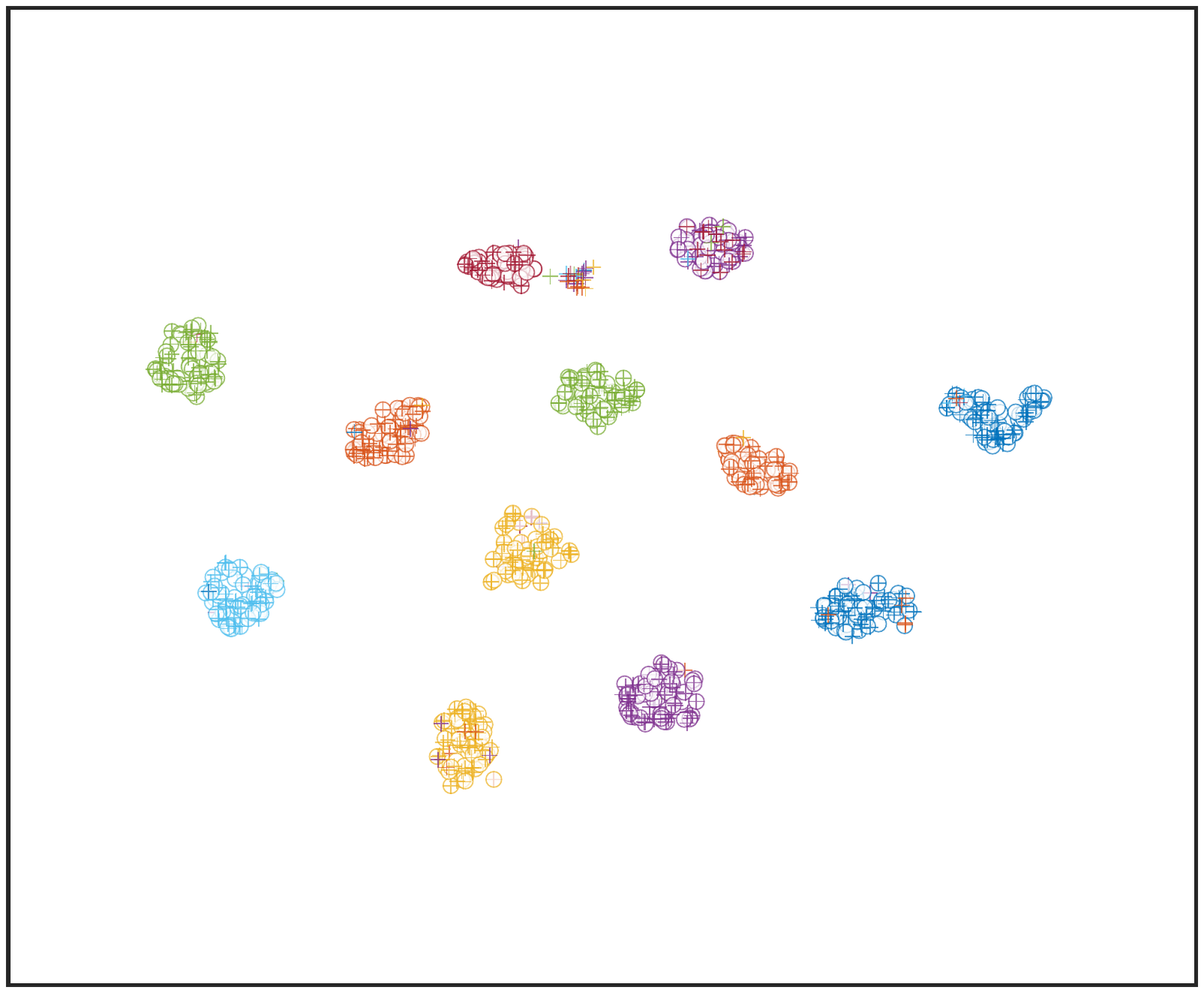}}
    \hfill
    \subfigure[ResNet (Office31) \label{fig:tsne_31_bf_class}]{\includegraphics[width=0.245\linewidth,height=77pt,trim=190 115 180 100,clip]{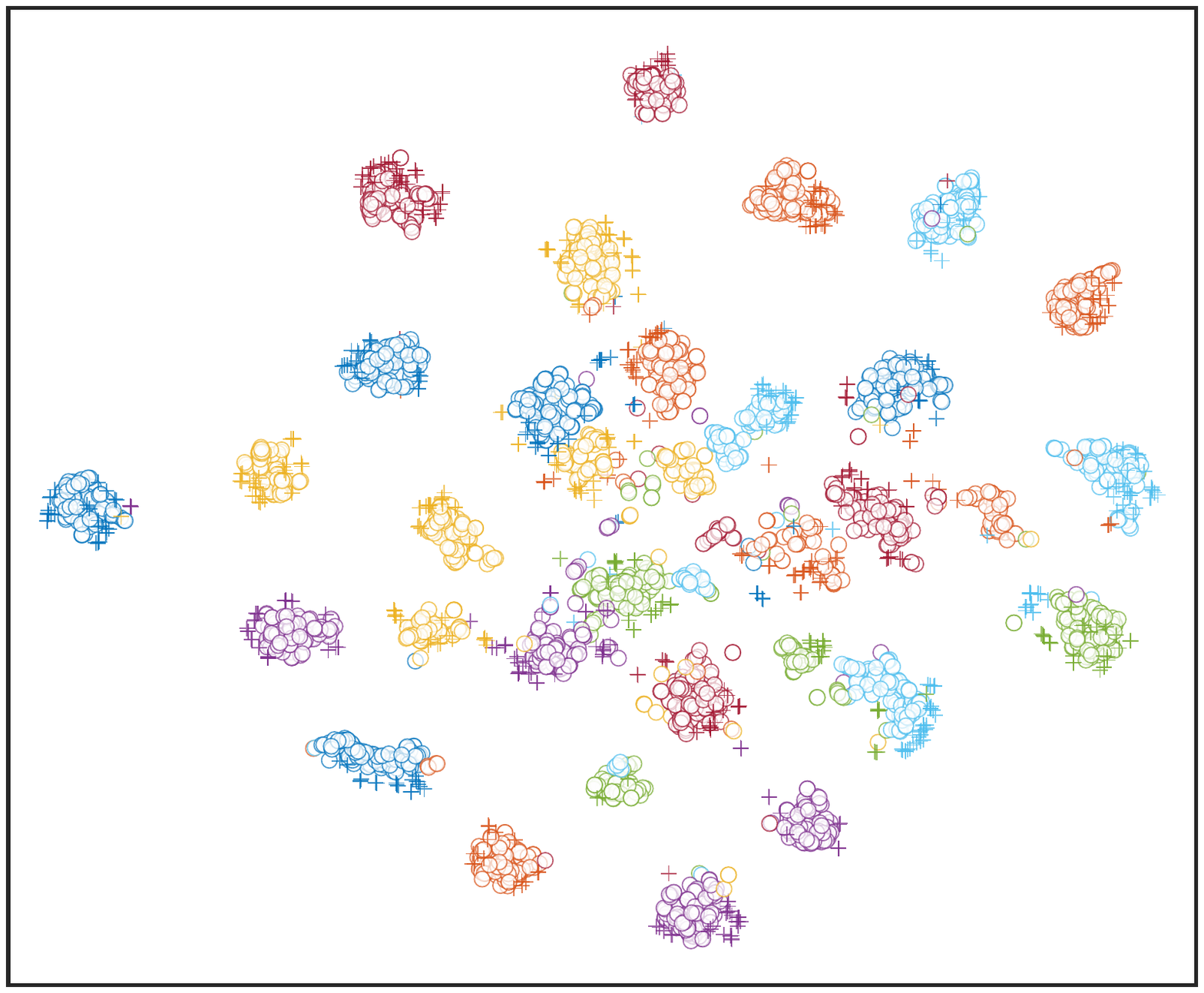}}
    \hfill
    \subfigure[GET (Office31) \label{fig:tsne_31_ad_class}]{\includegraphics[width=0.245\linewidth,height=77pt,trim=190 115 180 100,clip]{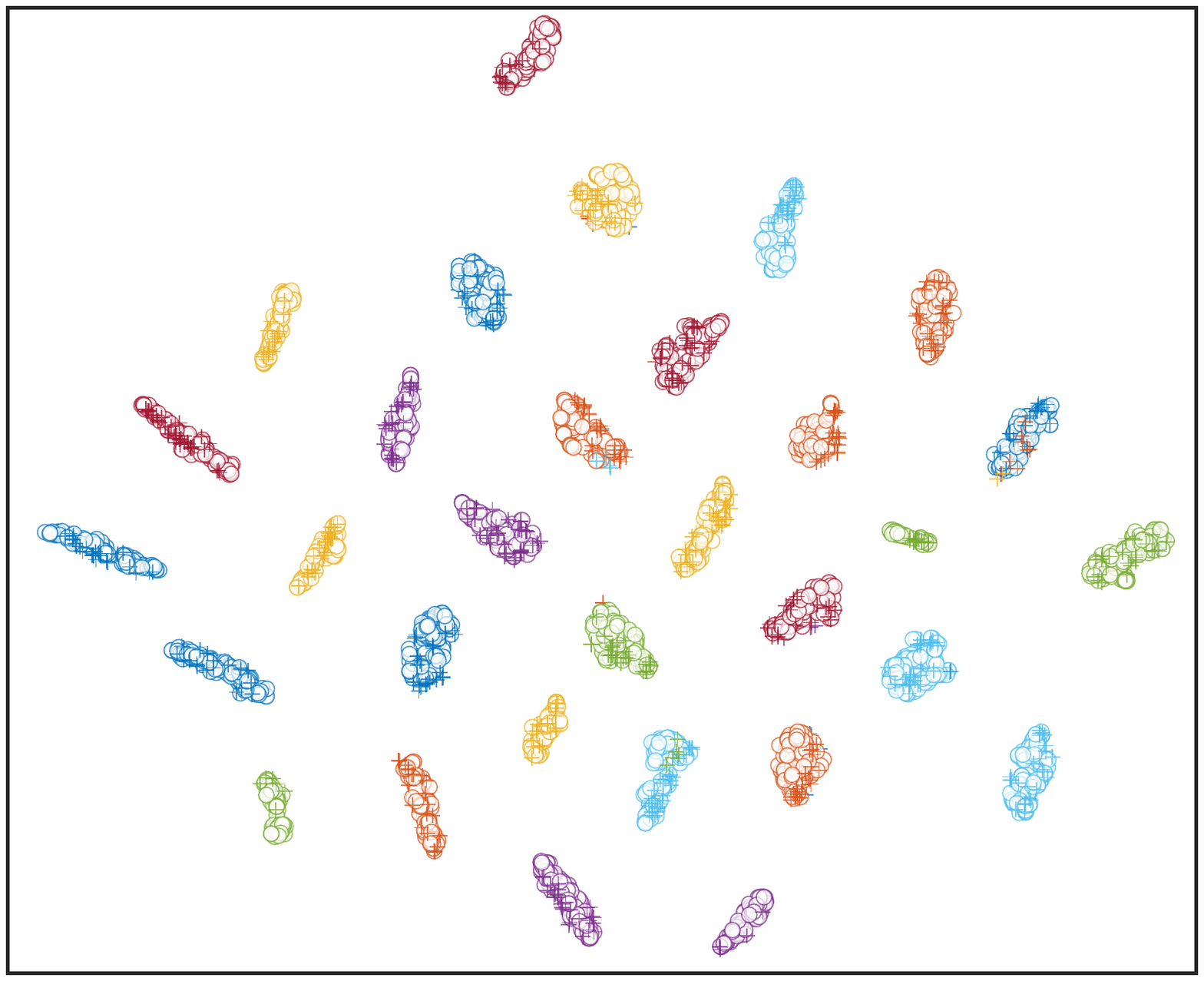}}
    \\
    \caption{\textbf{(a)}-\textbf{(b)}: Illustration of class orthogonality. \textbf{(c)}-\textbf{(d)}: The effect of hyper-parameters. \textbf{(e)}-\textbf{(h)}: t-SNE visualization of ResNet-50 and GET representations on ImageCLEF I$\rightarrow$P and Office31 A$\rightarrow$W, where `$\circ$'$\colon$source domain, `+'$\colon$target domain.}
    \label{fig:tSNE&Hyper&Ablation}
    \vskip -0.2in
\end{figure*}


We compare GET with several SOTA UDA methods: \textbf{DAN}~\cite{long2015learning}, \textbf{DANN}~\cite{ganin2015unsupervised}, \textbf{SAFN}~\cite{xu2019larger}, \textbf{ETD}~\cite{li2020enhanced}, \textbf{ALDA}~\cite{chen2020adversarial}, \textbf{DMP}~\cite{luo2020unsupervised}, \textbf{BNM}~\cite{cui2020towards}, \textbf{DWL}~\cite{xiao2021dynamic}.

\noindent \textbf{Office-Home}. Table \ref{tab:compare_home} shows the results of UDA task on Office-Home dataset. Compared with ImageCLEF and Office31, the adaptation tasks on Office-Home are more difficult because it has more categories and larger domain divergence. The average classification accuracy of GET is 68.6\%, which is higher than other SOTA UDA methods. It demonstrates that GET can learn the desired discriminability/transferability and improve accuracies on difficult tasks.

\noindent \textbf{ImageCLEF}. Table \ref{tab:compare_clef&31} (top) shows the results on ImageCLEF. Since GET learns both domain coherence and class orthogonality, there is a significant improvement over other methods which mainly focus on the domain alignment, \textit{e.g.}, DAN, DANN, ETD. Similar to GET, DWL considers the interaction between domain alignment and class discrimination, and achieves slightly better results on several tasks. But, GET learns more preferable geometry structure and achieves much higher accuracy on the most difficult C$\rightarrow$P task. These results show the importance of learning orthogonality in class subspaces. The average accuracy of GET is 91.3\%, which outperforms the other methods.

\noindent \textbf{Office31}. Table \ref{tab:compare_clef&31} (bottom) shows results on Office31. We observe that GET outperforms other comparison methods and achieves an accuracy of 89.5\% on the average. Specifically, there is a significant improvement from 73.4\% to 78.0\% on D$\rightarrow$A and from 71.6\% to 73.0\% on W$\rightarrow$A. Note that GET outperforms other two norm-based methods significantly, \textit{i.e.}, SAFN and BMM. It validates that GET can learn the stronger transferability and discriminability.

\subsection{Model Analysis}

\noindent \textbf{Toy Example}. To visualize the domain coherence and class orthogonality, we randomly select 3 classes from ImageCLEF I$\rightarrow$P and Office31 A$\rightarrow$W, and optimize GET on 3-dimensional space. As shown in Figure \ref{fig:3D_clefIP}-\ref{fig:3D_31AW}, the samples from the same class are gathered together on one axis (1-dimensional subspaces) while the clusters (axes) are nearly orthogonal. Besides, the source and target samples are well aligned and the class-level domain coherence is achieved. These results demonstrate that GET indeed learns a favorable geometry structure.

\noindent \textbf{Hyper-parameters}. We evaluate the hyper-parameters $\lambda_{DC}$ and $\lambda_{CO}$ on Office31 A$\rightarrow$W. The results are shown in Figure \ref{fig:hyper_clef}-\ref{fig:hyper_31}. It can be seen that the accuracies are sensitive to the hyper-parameters and decrease rapidly in the regions that $\lambda_{DC} > \lambda_{CO}$. In contrast, the results are robust when $\lambda_{DC}\leq \lambda_{CO}$; the highest accuracies are usually achieved when the parameters are balanced, \textit{i.e.}, $\lambda_{DC} = \lambda_{CO}$. These validate our theoretical results in Theorem \ref{thm:geometry_aware} which suggests that $\lambda_{DC}$ should be smaller to achieve better geometry structure. Besides, it demonstrates that the transferability and discriminability can be achieved simultaneously when $\lambda_{DC}$ and $\lambda_{CO}$ are properly chosen.


\noindent \textbf{Feature Visualization}. To evaluate the UDA performance at class-level qualitatively, we use the t-SNE~\cite{van2008visualizing} to visualize the representations extracted by ResNet-50 and GET on ImageCLEF I$\rightarrow$P and Office31 A$\rightarrow$W in Figure \ref{fig:tsne_clef_bf_class}-\ref{fig:tsne_31_ad_class}. As shown in \ref{fig:tsne_clef_bf_class}, there are several overlapping regions between the target ('+') and the source ('$\circ$') samples, but it is still visible that there exists a dataset shift and the misclassification may arise. For GET, the domains are well aligned and the clusters are clearly separated in \ref{fig:tsne_clef_ad_class}. For Office31 data in \ref{fig:tsne_31_bf_class}, there are many misaligned samples at class-level. In \ref{fig:tsne_31_ad_class}, our model alleviates this problem by grouping data from the same class as close as possible and separating data from different classes as far as possible. Besides, the results also demonstrate that GET learns the compact subspaces for different clusters.


\begin{table}
\begin{center}
\caption{Results of ablation study.}
\label{tab:ablation}
\renewcommand{\tabcolsep}{0.25pc} 
\vskip 0.05in
\begin{small}
\begin{tabular}{ccc|cccc}
\toprule
\multicolumn{3}{c|}{Constraint} & \multirow{2}{*}{ImageCLEF} & \multirow{2}{*}{Office31} & \multirow{2}{*}{Office-Home} &  \multirow{2}{*}{Avg.} \\
 $\mathcal{L}_{DC}$             & $\mathcal{L}_{CO}$ & $\mathcal{L}^t_{E}$          &                               &                           &                              &  \\
\midrule

 $\times$              & \checkmark  & \checkmark           &          90.6                     &        88.0                   &                 65.7             & 81.4  \\
 \checkmark            & $\times$    & \checkmark          &          87.7                     &           86.4                &                64.2              & 79.4 \\
 \checkmark  & \checkmark     & $\times$  &     91.2             &       88.7                    &             68.1                 & 82.7   \\
\checkmark &\checkmark              & \checkmark             &         \textbf{91.3}                      &       \textbf{89.5}                   &             \textbf{68.6}         & \textbf{83.1}  \\
\bottomrule
\end{tabular}
\end{small}
\end{center}
\vskip -0.2in
\end{table}

\noindent \textbf{Ablation Study}. We conduct ablation experiment to analyse the proposed constraints, \textit{i.e.}, domain coherence $\mathcal{L}_{DC}$ and class orthogonality $\mathcal{L}_{CO}$. The results are reported in Table \ref{tab:ablation}. The first two rows show that both the constraints improve the performance significantly on all datasets. Specifically, class orthogonality $\mathcal{L}_{CO}$ is more important when the domain divergence is small, \textit{e.g.}, ImageCLEF. However, when the domain divergence becomes larger, the class-level domain coherence $\mathcal{L}_{DC}$ is necessary to achieve the SOTA performance. The third row shows that the entropy objective $\mathcal{L}_{E}^t$ can be helpful when there are more classes and larger prediction uncertainty, \textit{e.g.}, Office31 and Office-Home. Besides, the model without $\mathcal{L}_{E}^t$ still achieves SOTA performance, which demonstrates that the contributions of the proposed constraints are significant. In the last row, GET is best on all datasets, which demonstrates that $\mathcal{L}_{DC}$ and $\mathcal{L}_{CO}$ are consistent and can benefit from each other.

\section{Conclusion}
In this paper, we propose a geometry-aware model called GET to enhance the transferability and discriminability simultaneously. This model starts from the rank property of matrix and then acquires the ability to learn coherent subspaces for domains and orthogonal subspaces for clusters via nuclear norm optimization. Based on GET, the derived theoretical results provide an insight into the norm-based learning for UDA. Specifically, we not only connect the transferability/discriminability enhancement and norm-based learning in UDA, but also show the possibility and condition to achieve these two abilities simultaneously. Extensive experiments validate the superiority of GET for UDA problem.


\bibliography{re}
\bibliographystyle{icml2022}


\end{document}